\newif\ifanonymous
\newif\ifdraft

\anonymoustrue
\draftfalse
\InputIfFileExists{localflags}{}

\documentclass{article}
\usepackage{ijcai17}
\usepackage{versions}

\pdfoutput=1 

\includeversion{conf}
\excludeversion{full}
\usepackage{amsthm} 
\usepackage{amsmath}
\usepackage{amssymb}
\usepackage{stmaryrd}
\usepackage{mathabx} 
\usepackage[utf8]{inputenc}
\usepackage{xspace}
\usepackage{proof}
\usepackage{color}
\usepackage{paralist}
\usepackage{thm-restate}
\usepackage{todonotes}
\usepackage{booktabs} 

\usepackage{hyperref}

\usepackage{listings} 
\usepackage{syntax} 

\makeatletter
\DeclareRobustCommand{\defeq}{\mathrel{\rlap{%
  \raisebox{0.3ex}{$\m@th\cdot$}}%
  \raisebox{-0.3ex}{$\m@th\cdot$}}%
  =}
\DeclareRobustCommand{\eqdef}{=\mathrel{\rlap{%
  \raisebox{0.3ex}{$\m@th\cdot$}}%
  \raisebox{-0.3ex}{$\m@th\cdot$}}%
  }
\makeatother

\newcommand{\set}[1]{\{\,#1\,\}}

\newtheorem{definition}{Definition}

\newtheorem{theorem}{Theorem}
\newtheorem{example}{Example}

\newtheoremstyle{mycase}{}{}{}{}{\bf}{.}{.5em}{\thmnote{#1:}~\normalfont
  #3}
\theoremstyle{mycase}

\numberwithin{subcase}{mycase}

\newtheoremstyle{component}{}{}{}{}{\itshape}{.}{.5em}{\thmnote{#3}#1}
\theoremstyle{component}

\newcommand{\ledot}{\mathrel{\ooalign{\hss\raise.200ex\hbox{$\cdot$}\hss\cr$\le$}}}
\newcommand{\gedot}{\mathrel{\ooalign{\hss\raise.200ex\hbox{$\cdot$}\hss\cr$\ge$}}}

\newcommand\bit{\lbrace0,1\rbrace}

\newcommand\setN{\mathbb N}

\newcommand{\calF}{\ensuremath{\mathcal{F}}\xspace}

\newcommand\calR{\ensuremath{\mathcal{R}}\xspace}
\newcommand\calS{\ensuremath{\mathcal{S}}\xspace}

\newcommand{\calU}{\ensuremath{\mathcal{U}}\xspace}
\newcommand\calV{\ensuremath{\mathcal{V}}\xspace}
\newcommand\calX{\ensuremath{\mathcal{X}}\xspace}
\newcommand\calY{\ensuremath{\mathcal{Y}}\xspace}

\newcommand{\mathcmd}[1]{{\normalfont\ensuremath{#1}}\xspace}

\newcommand{\mathfun}[1]{\mathcmd{\mathit{#1}}}

\newcommand{\textop}[1]{\relax\ifmmode\mathop{\text{#1}}\else\text{#1}\fi}
 
\mathchardef\mhyphen="2D

\newcommand{\proto}{\mathcmd{\Pi}} 

\newcommand{\pathp}{\mathfun{path}_\proto}

\DeclareMathOperator\E E

\newcommand{\Vres}{\calV^\mathit{res}}

\newcommand{\sM}{(M,\vec u)\vDash}

\begin{document}

\author{
    Robert K\"{u}nnemann\\
CISPA, Saarland University\\
Saarland Informatics Campus
} 
\title{Sufficient and necessary causation are dual}
\maketitle
\begin{full}
    \date{}
\end{full}

\begin{abstract}
    Causation has been the issue of philosophic debate since
    Hippocrates. Recent work defines actual
    causation in terms of Pearl/Halpern's causality framework,
    formalizing necessary causes
    (IJCAI'15). This has inspired causality notions in the
    security domain  (CSF'15), 
    which, perhaps surprisingly, formalize 
    sufficient causes instead.
    We provide an explicit relation between necessary and sufficient
    causes.
\end{abstract}


\paragraph{Notation}\label{sec:notation}%
Let $\setN$ be the set of natural numbers and assume that they begin at $0$. 
%
%
\begin{full}
For indicating that  function $f$ from a set $A$ to a set $B$ is
a partial function, we write $f : A \rightharpoonup B$. 
\end{full}%
%
 $r[v \mapsto \mathit{val}] := (r\setminus (v,r(v))) \cup
(v,val)$ is short-hand for the function mapping $v$ to $\mathit{val}$
and otherwise behaving like $r$. 
%
%
%
We write $\vec t$ for a sequence $t_1,\dotsc,t_n$ if $n$ is clear from the context and denote the $i$th element with $\vec t|_i$. We use $\vec a\cdot \vec b$ to denote concatenation of vectors $\vec a$ and $\vec b$.
%
%
%
%
We filter a sequence $l$ by a set $S$, denoted $l|_S$, by removing each element
that is not in $S$. 
\begin{full}
We use similar notation for the projection of a sequence $l=s_1,\ldots,s_n \in S^*$:
given a partial function $\pi \colon S \rightharpoonup T$, 
    $l|_\pi = (s_1,\ldots,s_{n-1})|_\pi\cdot \pi(s_n)$
    or 
    $(s_1,\ldots,s_{n-1})|_\pi$, if $\pi(s_n)$ undefined.
If some tree $\proto$ (for example a CoSP protocol, cf.  Section~\ref{sec:cosp}) is
given from context, we use $\pathp(v)$ to denote the path from the
root to $v$, including $v$.
\end{full}

\section{Causal model (Review)}

We review the causal framework introduced by Pearl
and Halpern~\cite{pearl-book,DBLP:conf/ijcai/Halpern15},
also known as the \emph{structural equations model},
which provides the notion of causality 
which we will investigate for the case of security
protocols.
The causality framework models how random variables influence
each other. The set of random variables is partitioned into
a set $\calU$ of \emph{exogenous} variables, variables that are
outside the model, e.g., in the case of a security protocol, the
attack the adversary decides to mount, and
a set $\calV$ of \emph{endogenous} variables, 
which are ultimately determined by the value of the
exogenous variables. 
We call the triple consisting of $\calU$, $\calV$ and function $\calR$
associating a range to each variable $Y\in\calU \cup \calV$
a \emph{signature}. 
A causal model on this signature defines the relation between
endogenous variables and exogenous variables or other endogenous
variables in terms of a set of functions.
\begin{definition}[Causal model]\label{def:causal-model} 
    A \emph{causal model} $M$ over a signature $\calS=(\calU,\calV,\calR)$
    is a pair of said signature $\calS$ and a set of functions
    $\calF=\set{F_X}_{X\in\calV}$ such that, for each $X\in\calV$,
    \[ F_X : (\bigtimes_{U\in\calU} \calR(U) ) \times
    (\bigtimes_{Y\in\calV\setminus \set{X}} \calR(Y) ) 
\to \calR(X) \]
\end{definition}

Each causal model subsumes a \emph{causal
network}, a graph with
a node for each variable in
$\calV$, and an edge from $X$ to $Y$ iff $F_Y$
depends on $X$.
If the causal graph associated to a causal model $M$ is
acyclic, then each
setting $\vec u$ of the variables in
$\calU$ provides a unique solution to the equations in $M$. 
All causal models we will derive in this paper have this property.
We call a vector setting the variables in $\calU$
a \emph{context}, and 
a pair $(M,\vec u)$ of a causal model and a setting \emph{a
situation}.

As hinted at in the introduction, the definition of causality follows
a counterfactual approach, which requires to answer `what if'
questions. 
\begin{full}
It is not always possible to do this by observing actual outcomes.
Consider the following example.

\begin{example}[Wet ground]\label{ex:wet-ground}
The ground in front of Charlie's house is slippery when wet. Not only
does it become wet when it rains; if the neighbour's sprinkler is
turned on, the ground gets wet, too. The neighbour turns on the
    sprinkler unless it rains. Let $R\in\calU$ be 1 if it rains, and
    0 otherwise, and consider the following equations for a causal
      model on $R$ and endogenous variables $W$, $S$ and $F$ with
      range $\bit$.
      \begin{align*}
          S & = \neg R & \text{(The sprinkler is on if it does not
          rain.)}\\
          W & = R \lor S & 
          \text{(The sprinkler or the rain wets the ground.)}\\
          F & = W & \text{(Charlie falls when the ground is
          slippery.)}
      \end{align*}
\end{example}
Clearly, the ground being wet is a cause to Charlie's falling,
but we cannot argue counterfactually, because the counterfactual case
never actually occurs: the ground is always wet. We need to intervene
on the causal model.
\end{full}

\begin{definition}[Modified causal model]\label{def:modified-causal-model} 
    Given a causal model 
    $M=((\calU,\calV,\calR),\calF)$, we define the 
    \emph{modified causal model}
    $M_{\vec X \leftarrow \vec x}$ over the signature
    $\calS_{\vec X}=(\calU,\calV \setminus \vec X, \calR|_{\calV
    \setminus \vec X})$ by setting the values of each variable in $\vec
    X$ to the corresponding element $\vec x$ in each equation $F_Y \in
    \calF$, obtaining $\calF_{\vec X \leftarrow \vec x}$. Then,
    $M_{\vec X \leftarrow \vec x} = (\calS_{\vec X}, F_{\vec X \leftarrow \vec x})$.
\end{definition}

\begin{full}
We can now define how to evaluate queries on
causal models w.r.t. to interventions on a vector of variables,
which allows us to answer `what if' questions.
\end{full}


\begin{definition}[Causal formula]\label{def:basic-causal-formula} 
    A \emph{causal formula} has the form
    $[Y_1 \leftarrow y_1, \ldots, Y_n \leftarrow y_n]\varphi$
    (abbreviated $[\vec Y \leftarrow \vec y]\varphi$), where
    \begin{itemize}
        \item $\varphi$ is a boolean combination of primitive events,
            i.e., formulas of form $X=x$ for $X\in\calV$,
            $x\in\calR(X)$,
        \item $Y_1,\ldots,Y_n\in\calV\cup\calU$ distinct,
        \item $y_i\in\calR(Y_i)$.
    \end{itemize}
    We write $(M,\vec u)\vDash [\vec Y \leftarrow \vec y]\varphi$%
    \begin{full}
    ($[\vec Y\leftarrow \vec y]\varphi$ 
    is true in a causal model $M$ given
    context $\vec u$) %
    \end{full}%
    if the (unique) solution to the
    equations in $M_{\vec Y\leftarrow \vec y}$ in the context $\vec u$ is an
    element of $\varphi$.
\end{definition}

%
Furthermore, we allow intervention on exogenous
variables. This is equivalent to a modification of the context:
$(M,\vec u)\vdash [ U_i \leftarrow u_i' ] \varphi$ equals
$(M,(\vec u|^0_{i-1}\cdot u_i' \cdot \vec u|^{i+1}_n))\vdash \varphi$. 
\begin{full}
In the example we regard, however, the context models
(among other things, like non-determinism in scheduling)
uncertainty about the behaviour of an external agent, the
adversary.
In this case, the situation is more accurately described by an
exogenous variable representing the adversary's motivation to, e.g., send
a certain message, and an endogenous variable representing the message
actually being sent. The endogenous variable mirrors the sending of
the message exactly, i.e., in the model, the adversary can send any
message he wants to, and the exogenous variable appears nowhere else.
As this is equivalent to a direct intervention on the exogenous
variable, we decide to keep the model concise and avoid this
mirroring.
\end{full}%
We review Halpern's modification~\cite{DBLP:conf/ijcai/Halpern15}
of 
Pearl and Halpern's definition of actual causes~\cite{DBLP:journals/corr/abs-1301-2275}.

\begin{definition}[actual cause\processifversion{conf}{+ necessary cause}]
    \label{def:actual-cause}\processifversion{conf}{\label{def:necessary-cause}}
    $\vec X = \vec x$ is a \emph{(minimal) actual cause} of $\varphi$ in
    $(M,\vec u)$ if the following three conditions hold.
    \begin{enumerate}
        \item [AC1.] $(M,\vec u)\vDash (\vec X = \vec x) \land
            \varphi$.
        \item [AC2.] There is a set of variables
            $\vec W$ and a setting $\vec x'$ of the variables in $\vec
            X$ such that if
            $(M,\vec u) \vDash (\vec W = \vec w)$,
            then
                $(M,\vec u) \vDash 
                    [\vec X \leftarrow \vec x', 
                    \vec W \leftarrow \vec w, 
                ] \neg \varphi$.
        \item [AC3.] $\vec X$ is minimal: No strict subset 
            $\vec X'$ of $\vec X$ satisfies AC1 and AC2.
    \end{enumerate}
    We say $\vec X$ is an \textbf{actual} cause for $\varphi$ if this is the case for some $\vec x$. %
    \begin{conf}
    For a weaker AC2 as follows, we speak of a (minimal) \textbf{necessary} cause
    \begin{itemize}
        \item [NC2.] There exists $\vec x'$ such that 
            $(M,\vec u) \vDash [\vec X\leftarrow \vec x'] \neg \varphi$.
    \end{itemize}
    \end{conf}
\end{definition}

\begin{full}
The first condition, AC1, requires that both the property and the
purported cause are facts that hold in the actual trace. In
particular, something that is not true cannot be an actual cause.
The second condition, AC2, is a generalisation 
of the intuition behind Lewis'
counterfactual. Let us first consider the special case
$\vec W = ()$, which corresponds exactly to Lewis' reasoning.

\begin{definition}[necessary cause]\label{def:necessary-cause} 
    Necessary causes are defined like actual causes (see
    Definition~\ref{def:actual-cause}), but with AC2 modified as
    follows:
    \begin{itemize}
        \item [AC2'.] There exists $\vec x'$ such that 
            $(M,\vec u) \vDash [\vec X\leftarrow \vec x'] \neg \varphi$.
    \end{itemize}
\end{definition}

Intuitively, AC2' requires that there is some intervention on the
cause $\vec X$ that negates $\varphi$, i.e., $\vec X=\vec x$ is
\emph{necessary} for $\varphi$ to hold. The condition AC2 is
a weakening of this notion to be able to capture causes that are
`masked' by other events. Suppose Charlie's neighbour in
Example~\ref{ex:wet-ground} is cruel and turns on the sprinkler
only if the ground is not wet from the rain:
\begin{align*}
    \mathit{WR} & = R & 
    S & = \neg \mathit{WR}\\
    W & = \mathit{WR} \lor S & 
    F & = W
\end{align*}
In the context where it rains ($R=1$), the ground gets wet ($\mathit{WR}=1$,
$W=1$) and Charlie falls $(F=1)$, but $\mathit{WR}=1$ is not
a necessary cause
for $F=1$, since $F=1$ even if $\mathit{WR}=0$. Hence the weakening of $AC2'$
to $AC2$: once we set $\vec W=S$, we can fix $S$ to its actual value
of $0$, reflecting that in the actual chain of events, the sprinkler
was in fact not turned on. Thus 
$(M,(1))\vDash [ \mathit{WR}\leftarrow 0, S\leftarrow 0] \neg(F=0)$.
The set of variables $\vec W$ is called \emph{contigency}, as it
captures the aspects of the actual situation under which $S=1$ is
a cause.

The third condition enforces minimality. In
Example~\ref{ex:wet-ground}, 
$(S,W)=(0,1)$ is not a cause of $F=1$ in the context $R=1$,
since $W=1$ is a cause by itself (while $S=0$ is not).
In the followup, 
we will sometimes discuss actual, necessary or sufficient (cf.
Definition~\ref{def:sufficient-cause}) causes
that are not necessarily minimal, in which case only AC1 and the
respective second conditions need to hold (AC2, NC2 or SF2), but not
AC3.

\end{full}

\section{Sufficient versus necessary causes}
\label{sec:duality}

The major difference underlying actual causes according to DGKSS~\cite{Datta2015a} 
and Pearl/Halpern~\cite{DBLP:conf/ijcai/Halpern15}
is that the former considers sufficient rather than necessary 
causes.\footnote{If $\vec X = \vec x$ is an actual cause under contingency
$\vec W = \vec w$, then $\vec X\cdot\vec W = \vec x\cdot \vec w$ is a necessary cause. 
Hence actual causes are parts of necessary causes.
}
We transfer this concept to Pearl's
causation framework as follows.
\begin{definition}[sufficient cause]\label{def:sufficient-cause} 
    Sufficient causes are defined like actual causes (see
    Definition~\ref{def:actual-cause}), but with AC2 modified as
    follows:
        SF2. For all $\vec z$,
            $(M,\vec u) \vDash [\calV \setminus \vec X\leftarrow \vec z] \varphi$.
\end{definition}

\noindent
In this section, we show that sufficient causes and necessary
causes (Definition~\ref{def:necessary-cause}) are dual to each other, and that
sufficient causes are in fact preferable, as they have a clearer
interpretation of what constitutes a part of a cause.

While several formalisations of sufficient causes were proposed
\cite{Datta2015a,gossler:hal-00924048}, so far they were never
related to necessary causes.
Strongest necessary conditions and weakest sufficient conditions
in propositional logic are known to be dual to each other,
however, sufficient and necessary causes are first-order predicates,
and there is no such result in first-order logic.
Even defining these notions is problematic~\cite{DBLP:journals/ai/Lin01}. 

We fix some finite set $\Vres\subset \calV$ and some ordering
$\set{V_1,\ldots,V_n} = \Vres$. We will see in the next section why
this restriction is useful.
Let $\overline X$ denote the bitstring representation of 
$X \subseteq \Vres$ relative to $\Vres$, i.e., 
$\overline X \defeq{} (1_X(V_1),\ldots,1_X(V_n))$.
We can now represent the set of necessary causes, or more generally,
any set of sets of variables $\calX=X_1,\ldots,X_m$, as a boolean formula in disjunctive
normal form (DNF) that is true whenever $\overline X$ is the bitstring
representation of $X\subseteq\Vres$ such that $X\in \calX$.
\[ 
(\overline X=\overline X_1 \lor \overline x=\overline X_2 \lor \cdots \lor \overline X= \overline X_m), \]
where $\overline X=\overline X_i$ is a conjunction 
$\bigwedge_{j\in \setN_{n}} \overline X|_j = \overline X_i|_j$.

\begin{theorem}[sufficient and necessary
    causes]\label{thm:sufficient-necessary}
    For $\calX$ the set of (not necessarily minimal) necessary causes,
    let $\overline \calX$ be the DNF representation of $\calX$.
    Then the set of 
    (not necessarily minimal)
    sufficient causes is represented by
    $\overline \calY$, which is obtained from $\overline \calX$
    by transforming $\overline \calX$ into CNF and switching $\lor$
    and $\land$. The same holds for the other direction.
\end{theorem}
\begin{proof}
    By Definition~\ref{def:necessary-cause}, NC2, we can rephrase the assumption as follows:
    $ \forall X. \exists x'.\ \sM [X\leftarrow x'] \neg\varphi
    \iff
      X\in\calX$. Now the right-hand side
    is equivalent to 
    $(\overline X = \overline X_1) \lor \cdots \lor (\overline X= \overline X_m)$.
    This is a boolean function over $\set{0,1}^n$. As any boolean
    function can be transformed into canonical CNF,
    the right-hand side can be expressed as
    $c_1 \land \cdots \land c_{k}$ with conjuncts $c_i$
        of form $\bigvee_{j\in\setN_{n}} (\neg) \overline X|_j$.
\begin{align*}
        \forall X. c_1 \land \cdots \land c_{k} &\iff \exists x'.\ \sM [X\leftarrow x'] \neg\varphi \\
        \intertext{Now we can negate both sides of the implication.}
        \forall X.  \neg c_1 \lor \cdots \lor \neg c_{k} &\iff \forall x'.\ \sM [X\leftarrow x'] \varphi.
\end{align*}
    We rename $X$ to $Z$ and $x'$ to $z'$. Let $\{^b / _a\}$ denote
    $b$ literally replacing $a$.
\begin{multline*}
        \forall Z.  \neg c_1 \left\{^{\overline Z} / _{\overline X}\right\}  \lor \cdots \lor \neg c_{k}\left\{^{\overline Z} / _{\overline X}\right\} 
        \\ \iff \forall z'.\ \sM [Z\leftarrow z'] \varphi
    \end{multline*}
        We can replace $Z$ by $X=\Vres\setminus Z$,
        as $Z \mapsto \Vres\setminus Z$ is
        is a bijection between the domain of $Z$ and the domain
        of $X$.
        Thus
        \begin{multline*}
        \forall X.  \neg c_1 \left\{^{\neg \overline X} / _{\overline X}\right\}  \lor \cdots \lor \neg c_{k}\left\{^{\neg \overline X} / _{\overline X}\right\} 
        \\ \iff \forall z'.\ \sM [\Vres\setminus X \leftarrow z'] \varphi. 
        \end{multline*}
    As each conjunct $c_i$ is a disjunction, the negation of $c_i$
    with $\overline X$ substituted by $\neg \overline X$ can be
    obtained by switching $\lor$ and $\land$. The resulting term is,
    again, a boolean formula in DNF, so $\overline \calX$ transforms
    into $\calX$ easily. 
    The reverse direction follows by first applying the above bijection and
    renaming backwards, and then following the first proof steps.
\end{proof}

To obtain the set of minimally sufficient causes from the set of
minimally necessary causes, one saturates the former by adding all
non-minimal elements (pick an element, and add its supersets by
iteratively switching all zeros to ones until a fixed point is reached)
and computes the set of (not-necessarily minimal) elements using the
above method.  The conversion to CNF can be performed via the
Quine–McCluskey algorithm, which is the obvious bottleneck in this
computation. Finally, the resulting set representation can be
minimised by removing all elements $X$ such that $\neg X \land Y$ for
some element in $Y$ (where $\neg$ and $\land$ are applied bitwise).

\begin{table}
    \centering
    \small
    \begin{tabular}{lp{2.8cm}p{2.8cm}}\toprule
        & necessary & sufficient \\\midrule
        \emph{Conj}$(1,1)$ & 
        $(A),(B)$ & 
        $(A,B)$ 
        \\
        \emph{Disj}$(1,1)$ & 
        $(A,B)$ & 
        $(A),(B)$ 
        \\
        \bottomrule
    \end{tabular}
    \caption{Comparison: set of all (minimal) necessary/sufficient causes. }\label{tab:comp}
\end{table}

\bibliography{references}
\bibliographystyle{named}

\end{document}